\documentclass{article}

\pdfoutput=1
\usepackage[T1]{fontenc}
\usepackage{lmodern}
\usepackage{slantsc}
\usepackage[protrusion=true,expansion=true]{microtype}
\usepackage{breakcites}
\usepackage{amsmath,amssymb,amsfonts,amsthm}
\usepackage{subcaption}
\usepackage{graphicx}
\usepackage{fullpage}
\usepackage{setspace}
\usepackage[backref=page]{hyperref}
\usepackage{color}
\usepackage{wrapfig}
\usepackage{tikz}
\usepackage{natbib}
\usetikzlibrary{decorations.pathreplacing}
\usepackage{algorithm}
\usepackage[noend]{algpseudocode}
\usepackage[framemethod=tikz]{mdframed}
\usepackage{xspace}
\usepackage{pgfplots}
\usepackage{framed}
\usepackage{thmtools}
\usepackage{thm-restate}
\usepackage{tabu}
\usepackage{fancyhdr}
\pgfplotsset{compat=1.5}

\newtheorem{theorem}{Theorem}[section]

\newtheorem{lemma}[theorem]{Lemma}

\newtheorem{definition}[theorem]{Definition}

\newenvironment{proofof}[1]{\begin{trivlist} \item {\bf Proof
#1:~~}}
  {\qed\end{trivlist}}

\newcommand{\namedref}[2]{\hyperref[#2]{#1~\ref*{#2}}}
\newcommand{\thmlab}[1]{\label{thm:#1}}
\newcommand{\thmref}[1]{\namedref{Theorem}{thm:#1}}
\newcommand{\lemlab}[1]{\label{lem:#1}}
\newcommand{\lemref}[1]{\namedref{Lemma}{lem:#1}}

\newcommand{\seclab}[1]{\label{sec:#1}}
\newcommand{\secref}[1]{\namedref{Section}{sec:#1}}

\newcommand{\alglab}[1]{\label{alg:#1}}
\renewcommand{\algref}[1]{\namedref{Algorithm}{alg:#1}}

\newcommand{\deflab}[1]{\label{def:#1}}
\newcommand{\defref}[1]{\namedref{Definition}{def:#1}}

\DeclareMathOperator*{\poly}{poly}

\newcommand{\ignore}[1]{}

\newif\ifnotes\notestrue %set this to true if notes are visible and to false (next line) if they should be hidden
\ifnotes
\newcommand{\samson}[1]{\textcolor{blue}{{\bf (Samson:} {#1}{\bf ) }} \marginpar{\tiny\bf
             \begin{minipage}[t]{0.5in}
               \raggedright S:
            \end{minipage}}}
\newcommand{\david}[1]{\textcolor{purple}{{\bf (David:} {#1}{\bf ) }} \marginpar{\tiny\bf
             \begin{minipage}[t]{0.5in}
               \raggedright D:
            \end{minipage}}} 
\newcommand{\shenghao}[1]{\textcolor{brown}{{\bf (Shenghao:} {#1}{\bf ) }} \marginpar{\tiny\bf
             \begin{minipage}[t]{0.5in}
               \raggedright D:
            \end{minipage}}} 
\else
\newcommand{\samson}[1]{}
\newcommand{\david}[1]{}
\newcommand{\shenghao}[1]{}
\fi

\makeatletter
\renewcommand*{\@fnsymbol}[1]{\textcolor{mahogany}{\ensuremath{\ifcase#1\or *\or \dagger\or \ddagger\or
 \mathsection\or \triangledown\or \mathparagraph\or \|\or **\or \dagger\dagger
   \or \ddagger\ddagger \else\@ctrerr\fi}}}
\makeatother

\providecommand{\email}[1]{\href{mailto:#1}{\nolinkurl{#1}\xspace}}

\definecolor{mahogany}{rgb}{0.75, 0.25, 0.0}
\definecolor{darkblue}{rgb}{0.0, 0.0, 0.55}
\definecolor{darkpastelgreen}{rgb}{0.01, 0.75, 0.24}
\definecolor{darkgreen}{rgb}{0.0, 0.2, 0.13}
\definecolor{darkgoldenrod}{rgb}{0.72, 0.53, 0.04}
\definecolor{darkred}{rgb}{0.55, 0.0, 0.0}
\definecolor{forestgreenweb}{rgb}{0.13, 0.55, 0.13}
\definecolor{greencss}{rgb}{0.0, 0.5, 0.0}
\definecolor{bleudefrance}{rgb}{0.19, 0.55, 0.91}

\hypersetup{
     colorlinks   = true,
     citecolor    = mahogany,
	 linkcolor	  = darkpastelgreen
}

\fancypagestyle{pg}
{
\lhead{}
\rhead{}
\cfoot{--\ \thepage\ --}

}

\AtBeginDocument{%
  \DeclareFontShape{T1}{lmr}{m}{scit}{<->ssub*lmr/m/scsl}{}%
}
\usepackage{natbib}
\usepackage{amsmath,amssymb,amsfonts,amsthm}
\newcommand{\wt}{\widetilde}

\newcommand{\R}{\mathbb{R}}

\renewcommand{\d}{\mathrm{d}}

\newcommand{\A}{\mathsf{A}}

\DeclareMathOperator*{\E}{{\mathbb{E}}}

\DeclareMathOperator{\vect}{vec}

\newif\Ifnotes\notestrue %set this to true if notes are visible and to false (next line) if they should be hidden
\usepackage[textsize=tiny]{todonotes}

\begin{document}

\title{Towards Sampling Data Structures for Tensor Products in Turnstile Streams}
\author{
Zhao Song\thanks{University of California, Berkeley. 
E-mail: \email{magic.linuxkde@gmail.com}.
}
\and
Shenghao Xie\thanks{Texas A\&M University. 
E-mail: \email{xsh1302@gmail.com}.
Supported in part by NSF CCF-2335411.
}
\and
Samson Zhou\thanks{Texas A\&M University. 
E-mail: \email{samsonzhou@gmail.com}.
Supported in part by NSF CCF-2335411. 
The author gratefully acknowledges funding provided by the Oak Ridge Associated Universities (ORAU) Ralph E. Powe Junior Faculty Enhancement Award.
}
}
\maketitle

\begin{abstract}
This paper studies the computational challenges of large-scale attention-based models in artificial intelligence by utilizing importance sampling methods in the streaming setting. Inspired by the classical definition of the $\ell_2$ sampler and the recent progress of the attention scheme in Large Language Models (LLMs), we propose the definition of the attention sampler. Our approach significantly reduces the computational burden of traditional attention mechanisms. We analyze the effectiveness of the attention sampler from a theoretical perspective, including space and update time. Additionally, our framework exhibits scalability and broad applicability across various model architectures and domains. 
\end{abstract}

\section{Introduction}

In recent years, the field of artificial intelligence has witnessed a significant paradigm shift with the advent of attention-based models, particularly in the domains of natural language processing and computer vision \citep{vsp+17,dclt18,log+19,ydy+19,bmr+20,zrg+22,cnd+22,tli+23,tms+23,adobe_firefly,m23}. 
At the heart of these models lies the attention mechanism \citep{vsp+17}, which is a powerful tool for enhancing the performance of deep learning networks. 
In particular, the attention mechanism enables models to focus on relevant parts of the input data, thereby facilitating context-aware processing. 

However, as these models scale in size and complexity \citep{Zeng+24,Reid+24,ZhangCH+24,Dubey+24,Abdin+24}, the computational demands of the attention mechanism increase significantly, posing challenging barriers towards efficient scalability \citep{Fu24}.
Specifically, traditional attention mechanisms used in Transformer models \citep{vsp+17} require computing attention weights across all elements of the input sequence, leading to a \emph{quadratic} increase in computational complexity with respect to the sequence length \citep{as23,kmz23,hjk+23,zhdk23,as24,as24_neurips,as25_rope}. 
This computational burden becomes particularly pronounced in large-scale applications, hindering the usage of attention-based models in resource-constrained settings and limits their real-time processing capabilities. 
%Furthermore, the high computational cost also exacerbates the environmental impact of training and deploying these models, due to increased energy consumption and carbon footprint.

To deal with this problem, the core question we ask in this paper is: 
\begin{center}
{\it
Instead of computing all entries, can we recover the most important ones in efficient space and time?
}
\end{center}

\paragraph{Attention samplers.}
We adopt the classical idea of sampling a dataset, selecting ``important'' items to represent the entire dataset.
Sampling is a central and effective technique for analyzing large-scale datasets, which has broad application in the field of big data~\citep{Vitter85,GemullaLH08,CohenDKLT11,CohenDKLT14}, including network traffic monitoring~\citep{MaiCSYZ06,HuangNGHJJT07,ThottanLJ10}, database management~\citep{HaasS92,Haas16,CohenG19}, and data summarization~\citep{FriezeKV04,AggarwalDK09,MahabadiIGR19, IndykMGR20,MahabadiRWZ20}.
A well-known example is the $\ell_2$ sampler first asked by \cite{CormodeMR05} and subsequently studied by \cite{MonemizadehW10,AndoniKO11,JowhariST11,JayaramW21,PettieW25,SwartworthWZ25,WoodruffXZ25}: given a vector $x \in \mathbb{R}^n$, we sample an index $i \in [n]$ with probability roughly $\frac{x_i^2}{\|x\|_2^2}$.

To address the challenges in implementing large-scale attention schemes, we seek to sample the most important coordinates in attention computation, reducing computational overhead and computer storage. 
Inspired by the classical definition of the $g$-sampler on a vector, we propose the following attention sampler, which is thoroughly investigated in this paper.

\begin{definition}[Attention sampler]
Given matrix $A \in \R^{n \times d}$, vector $x \in \R^d$, and a distribution function $g$, the attention sampler samples index $i \in [n]$ with probability
%\begin{align*}
$
    p_i = \frac{g((Ax)_i)}{\sum_{j=1}^n g((Ax)_j)}+\frac{1}{\poly(nd)}
$.
\end{definition}
%By strategically sampling key elements from the input data, our approach significantly reduces the computational overhead associated with the attention mechanism.
%, while maintaining, or even enhancing, the model's performance.
%This paper presents a comprehensive exploration of our proposed sampling techniques, detailing the underlying principles, implementation strategies, and the resultant gains in computational efficiency.

The motivation of our definition lies in the internal structure of the attention mechanism.
Given input matrices $A_1$ and $A_2$, the linear attention matrix is defined as
$A_1 X A_2^\top$, where $X = QK^\top$ is the fused key and query matrix. 
For linear self-attention, $A_1$ and $A_2$ are identical.
Utilizing a well-known tensor product construction \citep{as24}, we simplify the expression of the attention matrix to a matrix-vector product.
Let $A = A_1 \otimes A_2$ and let $x = \vect(X)$. 
Then the vectorized linear attention matrix turns out to be $Ax = \vect ( A_1 X A_2^\top )$. 
Here, $\vect$ denotes the vector representation of a matrix by concatenating the rows. 
Therefore, our attention sampler detects the dominant entry in the linear attention matrix, providing an effective approximation of the attention scheme.

%Based on the above formulation, we put forward the following definition of attention samplers, which is thoroughly investigated in this paper.
Given unlimited space and time, the sampling problem is trivial since one can compute each entry explicitly and sample an index with the corresponding probability.
However, as mentioned earlier, we are not granted unlimited resource in real-world applications, for instance, in resource-constrained settings or in real-time processing.
This motivates us to investigate the attention sampler in the streaming model, where the input matrix $A$, the weight vector $x$, or both $A$ and $x$ arrive sequentially in a data stream, and the goal is to report a valid sample \emph{at all times} using efficient space and update time.
In this paper, we study turnstile data streams, where updates to the underlying data can either increase or decrease the corresponding values at each time. 

Indeed, as databases handle increasingly vast and dynamic real-time data, the streaming model has emerged as a vital framework for designing algorithms to process massive, constantly evolving datasets. 
Examples include real-time analysis of social media streams, sensor data for smart infrastructure, live video processing, detection of distributed denial of service (DDoS) attacks, and efficient indexing and querying in large-scale databases. 
In this work, we combine the streaming model with attention mechanisms and construct novel efficient attention samplers, which identify the critical coordinates in attention computation.
Our contributions can be summarized as follows:
 
\begin{itemize}
    \item For the softmax distribution $\langle \exp(Ax), {\bf 1}_n \rangle^{-1} \exp(Ax) $, we prove an $\Omega(n)$ space streaming sampler algorithm lower bound. (See \thmref{thm:softmax_lower_bound})
    \item As the softmax distribution has a strong lower bound, we then provide upper bounds for polynomial type samplers, i.e., $\ell_2$ sampling from $Ax$:
    \begin{enumerate}
        \item For updating $A$ and fixed $x$, our sampler takes $d \poly \left(\frac{1}{\epsilon},n\right)$ bits of space and update time (see \thmref{thm:update_A_fix_x}).
        \item For updating $A$ and fixed $x$, our sampler takes $d \poly \left(\frac{1}{\epsilon},n\right)$ bits of space and $O(1)$ update time (see \thmref{thm:fix_A_update_x}).
        \item For updating both $A$ and $x$, our sampler takes $d \poly \left(\frac{1}{\epsilon},n\right)$ bits of space and update time (see \thmref{thm:update_A_update_x}).
    \end{enumerate}
    \item For updating both $A$ and $x$, we also provide a lower bound of $\Omega(d)$ space (see \thmref{thm:update_A_update_x_lower_bound}).
    \item Toward tensor generalization, where we have updating $A_1 \in \R^{n \times d}$ or $A_2 \in \R^{n \times d}$ for $A = A_1 \otimes A_2 \in \R^{n^2 \times d^2}$ and fixed $x \in \R^{d^2}$, we sample $(i_1,i_2) = i \in [n^2]$ approximately according to the $\ell_2$ sampling distribution on $Ax \in \R^{n^2}$ using $O(nd)$ space, $O(n)$ update time (see \thmref{thm:tensor_sampler}). Note that the trivial result takes $O(n^2)$ space. 
\end{itemize}

%We implement a simple experiment to demonstrate our theoretical results. We do a comparison between our $\ell_2$ sampler and the uniform sampler on recovering the large coordinates of a given sparse attention matrix. The $\ell_2$ sampler succeeds with significantly smaller sampled indices, which proves its ability to detect large coordinates using relatively restricted memory.

\paragraph{Hardness of softmax attention.}
Our lower bound in the first result demonstrates the hardness for computing or approximating the softmax attention.
This aligns with the lower bound in \cite{as23}, which showed that approximating softmax attention up to small entry-wise error requires subquadratic time in $n$ assuming the Strong Exponential Time Hypothesis. 
These challenges motivate us to explore polynomial attention mechanisms. 
To that end, previous work has investigated the performance of polynomial attention from both theoretical and empirical perspectives.
For instance, the PolySketchFormer \citep{kmz23} demonstrated that polynomial attention achieves model quality comparable to softmax attentions with efficient low-dimensional approximations. 
Furthermore, polynomial attention schemes perform competitively in various vision and NLP tasks, including the linear attention in \cite{KoohpayeganiP24} and the polynomial attention in \cite{szjzl24}.
Building on these insights, we obtain efficient polynomial attention samplers in the streaming model, whose space and update time have no dependence on $n$ factors, effectively recovering the key components in the polynomial attention matrix. 

\paragraph{Streaming attention mechanism.}
Our polynomial attention samplers work in the streaming model, which matches the core idea of streaming Large Language Model (LLM) introduced and studied by \cite{XiaoTCHL24}, and recently gained increasing focus in LLM research and big-data analysis \citep{StratiMPTK24,YaoL024,Shikhar+25,XiaoTZGYTF025}.
The motivation is from long (or infinite) sequence generation, e.g., a chat bot having a day-long conversation.
When we apply LLMs in these scenarios, we often encounter a efficiency-performance trade-off: during the decoding stage, attention-based methods cache all Key and Value (KV) pairs, which requires excessive memory usage; in contrast, under restricted memory, the performance collapses once the sequence length exceeds the cache size. 
To deal with these drawbacks, \citep{XiaoTCHL24} trained the models with a finite attention window to work on text of infinite length.
Unlike \cite{XiaoTCHL24}, our model is dynamic and data-driven, supporting both model weights and input tokens to constantly change.
We note that our sampler provides a correct attention sample at all times using efficient space.
Thus, we identify the important coordinates in attention computing without probing each KV pair, which has high potential in enhancing the performance of streaming LLMS.

\paragraph{Sparse attention mechanism.}
%Our proposed methods have various important applications in LLM. 
Another practical relevance of our attention sampler is sparse attention mechanism.
The attention matrix has been shown to be naturally sparse empirically and theoretically (see e.g. \citep{dsy24}). 
Based on this observation, researchers seek to reduce computation by sampling the attention layers \citep{cgrs19,kkl20,wlk+20,as23,bsz23,dms23,LaiLLMZ25,XiaoTZGYTF025,Zhang+25}. 
In general, they construct a \emph{sparse mask} that selects the importance entries in the attention multiplications while others are zeroed out. 
Then, they compute the partial attention corresponding to those in the sparse mask.
Specifically, \cite{XiaoTZGYTF025} explores the sparse attention with streaming heads.
Our attention sampler recovers large coordinates from the attention matrix given specific streamed inputs and weights. 
Thus, the sampler serves as an efficient subroutine in their sparse-attention sampling schemes, evaluating and enhancing the effectiveness of their construction of the sparse mask. 

\paragraph{Streaming algorithms.}
In addition, our sampler can be integrated into inner product computation (see e.g. \cite{WoodruffZ21}), which is a cornerstone for model training and attention computation.
In fact, classical $\ell_p$ samplers also serve as black-box subroutines in many other streaming algorithms, including finding heavy hitters, $F_p$ moment estimation, and cascaded norm approximation \citep{AndoniKO11,JowhariST11,JayaramW21,WoodruffZ21}.
Therefore, our attention sampler can be applied to discover essential properties of the attention scheme, e.g., the norm of the attention matrix.

\section{Related Work}
In this section, we discuss a number of related works and their relevant implications on our results. 

\paragraph{On sampling.} %%% TCS papers
% Given a vector $v\in U^n$ whose coordinates are elements from a universe $U$ and a non-negative weight function $W:U\to\mathbb{R}^{\ge 0}$, a fundamental goal is to return an index $i\in\{1,\ldots,n\}$ with probability proportional to $W(v_i)$. 
% The definition of $U$ permits settings such as $U=\mathbb{R}^d$, so that each coordinate is a row of a matrix or a $d$-dimensional point, or $U$ may be a subset of the set of all matrices or tensors. 
% In perhaps the most well-studied setting, each coordinate is a real number, so that $U=\mathbb{R}$ and the weight function is chosen from the class $W(x)=|x|^p$ for $p\ge 0$. 
% The problem is particularly interesting when the vector $v\in U^n$ is implicitly defined through a data stream, i.e., a sequence of $m$ updates to the coordinates of $v$, and the goal is to perform the sampling procedure using space sublinear in $n$ and $m$. 
Given a vector $v \in \R^n$ and a distribution function $g$, recall that the classical $g$-sampler samples index $i \in [n]$ with probability
%\begin{align*}
$
    p_i = \frac{g(v_i)}{\sum_{j=1}^n g(v_j)}
$.
A well-known example is the $\ell_p$ sampling defined by $g(z) = |z|^p$ for $p\ge 0$.
The existence of such a $\ell_p$ sampler algorithms first posed as a question by~\cite{CormodeMR05} in 2005. 
\cite{MonemizadehW10} partially answered this question in the affirmative by giving an $\ell_p$ sampler using polylogarithmic space for $p\in[1,2]$, although the sampling probabilities were distorted by a multiplicative $(1+\epsilon)$ factor and an additive $\frac{1}{\poly(n)}$ factor. 
We note that the sampler is perfect if there is no $\epsilon$-multiplicative distortion; it is truly perfect if there is no additive distortion, i.e., the sampling probability is exact.
The space requirements of the algorithm were subsequently improved \citep{AndoniKO11,JowhariST11} and extended to other choices of index domain $U$ and weight function $W$~\citep{CohenG19,MahabadiRWZ20,MahabadiWZ22}, while retaining a multiplicative distortion in the sampling probability. 
Surprisingly, \cite{JayaramW21} showed that it is possible to achieve no perfect samplers while using polylogarithmic space, while conversely \cite{JayaramWZ22} showed that truly perfect samplers would require linear space, essentially closing the line of work studying the space complexity of $\ell_p$ samplers for $p \in [1,2]$.
It should be noted however, achieving such guarantees (no additive distortion) in sub-polynomial update time while retaining the space guarantees remains an intriguing open question \citep{JayaramWZ22}. 
For the other regime of $p > 2$, recently, \cite{WoodruffXZ25} complemented the results by providing efficient perfect $\ell_p$ samplers for $p > 2$.
\cite{SwartworthWZ25} achieved perfect samplers with polylogarithmic update time for $p<2$, improving on the previous update time.
For a more comprehensive background on samplers, we refer to the survey by \cite{CormodeJ19}.

\paragraph{On tensors.} %%% Tensor

In the realm of tensor decomposition, the canonical polyadic (CP) decomposition, specifically the CANDECOMP/PARAFAC method, stands out for its unique ability to break down tensors into rank-1 tensors in a singular way, distinct from matrix decomposition \citep{h70, swz16}. This method, having applications in computational neuroscience, data mining, and statistical learning \citep{wtsa15}, emphasizes the rigidity and uniqueness of tensor decomposition. Earlier studies \citep{t10, ptc13, cv14, hnh+13, kphf12, wlsh14, bs15} have delved into efficient tensor decomposition methods. Subsequent works introduced methods for fast orthogonal tensor decomposition using random linear sketching techniques \citep{wtsa15} and explored symmetric orthogonally decomposable tensors' properties, integrating spectral theory \citep{r16, r17b}. Additionally, importance sampling for quicker decomposition was proposed \citep{swz16}. \citep{dgs23} studies the tensor cycle low rank approximation problem.

In algebraic statistics, tensor decompositions are linked to probabilistic models, particularly in determining latent variable models' identifiability through low-rank decompositions of specific moment tensors \citep{amr09, aprs11, rs12}. Kruskal's theorem \citep{k77} was pivotal in ascertaining the precision of model parameter identification. However, this approach, assuming an infinite sample size, does not provide the minimum sample size for learning model parameters within given error bounds. 
A more robust uniqueness guarantee is needed to ensure that the low-rank decomposition of an empirical moment tensor approximates that of an actual moment tensor, thus offering more insight into empirical moment tensors' decomposition.

\paragraph{On sketching.}

The application of sketching and sampling techniques in numerical linear algebra has been remarkably effective, revolutionizing a broad spectrum of core tasks. These methods are crucial in linear programming (LP), as evidenced by \cite{cls19,jswz21,y20,gs22}, and have significantly impacted tensor approximation \citep{swz19, mwz22, dgs23}. Sketching and sampling techniques also have been widely applied in matrix completion \citep{gsyz23}, matrix sensing \citep{qsz23,dls23}, submodular function maximization \citep{qsw23}, dynamic sparsification \citep{djs+22}, dynamic tensor product regression \citep{rsz22}, and semi-definite programming \citep{syyz22-lichen}. Additionally, sketching has been pivotal in iterative sparsification problems \citep{sxz22}, adversarial training \citep{gqls23}, kernel density estimation \citep{qrs+22}, solving the distance oracle problem \citep{dswz22}, and empirical risk minimization \citep{lsz19, qszz23}. Its applications furthermore extends to relational databases \citep{qjs+22} and  Large Language Model (LLM) research \citep{dms23, dls23, gsy23, lsz23}.

\paragraph{On theoretical attention.} %%% ML papers.

A comprehensive body of research, including studies \citep{cgrs19,kkl20,wlk+20,dkod20,kvpf20,cdw+21,cdl+22,zhdk23,as23,bsz23,dms23,kmz23,as24,hjk+23,ag23,mda22,as24_neurips,as25_rope,as25_weight}, has progressively shed light on the complexities and optimization of attention matrix computation. This exploration has been further enriched by insights into the effectiveness of attention mechanisms in Transformers \citep{dgv+18, vbc20, zkv+20, egkz21, szks21, wcm21, dsx23, dlms23}. Among these, \cite{zpga23} revealed the adeptness of mid-scale masked language models in identifying syntactic elements, paving the way for innovations like partial parse tree reconstructions. Inspired the exponential mechanism in attention structure, \cite{gms23} provide an analysis which shows exponential regression within the over-parameterized neural tangent kernel framework can converge. In the over-constrained setting, several work show the convergence for attention inspired regression problem \citep{lsz23,dls23}.

\paragraph{Organiation of the rest of the paper.}
In \secref{sec:preli}, we provide some standard notations and definitions in literature. In \secref{sec:exp}, we study the exponential sampler. 
In \secref{sec:upper_A_x}, we study the streaming upper for the $\ell_2$ sampling problem, i.e., sampling coordinates from a vector $Ax$, where $A$ and $x$ may be updated across a data stream. 
In \secref{sec:lower_A_x}, we present lower bounds for the same $\ell_2$ sampling problem. 
In \secref{sec:tensor}, we discuss the tensor sampling problem. 
%Finally, in Section~\ref{sec:experiment}, we provide experimental verification to our theoretical results.
%In Section~\ref{sec:conclusion}, we provide a conclusion. In Section~\ref{sec:limitation}, we discuss the limitaiton of this paper. In Section~\ref{sec:impact}, we discuss the broader impact of our work.

\section{Preliminaries}\seclab{sec:preli}

For any positive integer $n$, we use $[n]$ to denote the set $\{1,2,\cdots, n\}$. 
We use $\E[ \cdot ]$ to denote the expectation. We use $\Pr[ \cdot ]$ to denote the probability. 
We use ${\bf 1}_n$ to denote a length-$n$ vector where all the entries are ones. 
Given two length-$n$ vectors $x,y\in\mathbb{R}^n$, we use $\langle x,  y \rangle$ to denote the inner product between $x$ and $y$, i.e, $\langle x, y \rangle : = \sum_{i=1}^n x_i y_i$. 
For a vector $x\in \R^n$, we use $\exp(x) \in \R^n$ to denote a vector that has length $n$ and the $i$-th entry is $\exp(x_i)$.  
For a matrix $A$, we use $\exp(A)$ to denote the matrix that $(i,j)$-th coordinate is $\exp(A_{i,j})$. 
For a vector $x$, we use $\| x \|_2 : = ( \sum_{i=1}^n x_i^2 )^{1/2} $. 
We use $\| x \|_1 : = \sum_{i=1}^n |x_i|$. 
We use $\| x \|_0$ to denote the $\ell_0$ norm of $x$, which is the number of nonzero entries in $x$. 
We use $\| x \|_{\infty}$ to denote the $\ell_{\infty}$ norm of $x$, which is $\max_{i\in [n]} |x_i|$.  

Let $n_1, n_2, d_1, d_2$ be positive integers. 
Let $A \in \R^{n_1 \times d_1}$ and $B \in \R^{n_2 \times d_2}$. 
We define the Kronecker product between matrices $A$ and $B$, denoted $A \otimes B \in \R^{n_1 n_2 \times d_1 d_2}$, by $(A \otimes B)_{(i_1 - 1) n_2 + i_2, (j_1-1)d_2+j_2}$, to be equal to $A_{i_1,j_1} B_{i_2,j_2}$, where $i_1 \in [n_1], j_1 \in [d_1], i_2 \in [n_2], j_2 \in [d_2]$. 

We use $\poly(n)$ to denote $n^C$ where $C>1$ is some fixed constant. For any function $f$, we use $\wt{O}(f)$ to denote $f \cdot \poly(\log f)$. 
For two sets $A$ and $B$, we use $A \cap B$ to denote their intersection. 
We use $|A \cap B|$ to denote the cardinality of $A \cap B$. 
We use $A \cup B$ to denote the union of $A$ and $B$.

\paragraph{TensorSketch.}
We next define TensorSketch \citep{p13}, which has been extensively used in many sketching and optimization problems~\citep{dssw18,swz19,djs+19,akk+20,swyz21,szz21,sxz22,z22,syz23_sdp}. 
\cite{sxz22} defined {\sf TensorSparse} by composing Sparse embedding \citep{nn13,c16} with a tensor operation \citep{p13}.

\begin{definition}[{\sf TensorSparse}, see Definition~7.6 in \cite{sxz22}]
\deflab{def:tensor_sparse}
Let $h_1,h_2:[n] \times [s]\rightarrow [m/s]$ be $O(\log 1/\delta)$-wise independent hash functions and let $\sigma_1,\sigma_2:[n ]\times [s]\rightarrow \{\pm 1\}$ be $O(\log 1/\delta)$-wise independent random sign functions. Then, the degree two tensor sparse transform, $S:\R^n \times \R^n \rightarrow \R^m$ is given as: 
\begin{align*}
   & ~  R_{r,(i,j)} =  \exists k\in [s]: \sigma_1(i,k)\sigma_2(j,k)/\sqrt{s}\cdot %\\
   {\bf 1}[ ((h_1(i,k)+h_2(j,k))~\text{mod~}m/s)+(k-1)m/s=r]
\end{align*}
\end{definition}
For $s=1$, the above definition becomes TensorSketch \citep{p13}.

\section{Exponential Sampler}
\seclab{sec:exp}
In this section, we define and consider exponential samplers.  
We then show strong space lower bounds for achieving such a data structure when the input dataset arrives in a data stream. 

Let us firstly describe the offline version: 
\begin{definition}[Exponential sampler]
Given matrix $A \in \R^{n \times d}$ and $x \in \R^d$, the goal is to sample index $i \sim [n]$ with probability
%\begin{align*}
$
    p_i = \langle \exp(Ax) , {\bf 1}_n \rangle^{-1} \cdot \exp( Ax )_i,
$
%\end{align*}
where ${\bf 1}_n$ denotes a length-$n$ vector, $\exp(Ax) \in \R^n$ denotes a length-$n$ vector with $\exp(Ax)_i = \exp( (Ax)_i )$, and $\exp(z)$ is the usual exponential function. 
%Here
%\begin{itemize}
    %\item Let ${\bf 1}_n$ denote a length-$n$ vector 
    %\item Let $\exp(Ax) \in \R^n$ denote a length-$n$ vector where $\exp(Ax)_i = \exp( (Ax)_i )$
    %\item $\exp(z)$ is exponential function. (We can either treat it as $e^z$ or $2^z$, it depends on which one is more convenient.)
%\end{itemize}
\end{definition}
Now, consider $y = Ax \in \mathbb{R}^n$, where either $A$ or $x$, or both are arriving in a data stream. 
% Note that at the end of the stream, we only need to sample one index $i\in[n]$. 
% On the other hand, there are three possibilities for streaming version:
% \begin{itemize}
%     \item Both $A$ and $x$ arrive in streaming fashion
%     \item $A$ is fixed but $x$ arrives in streaming fashion
%     \item $x$ is fixed but $A$ arrives in streaming fashion
% \end{itemize}
% We consider each of these cases separately. 
We use the following definition for each of the various cases:
\begin{definition}\deflab{def:exponential_sampler}
Let $C >0$ be any fixed constant and let $C_0 \in [n^{-C}, n^C]$. 
Let $y$ be a vector. 
Then the exponential sampler outputs an index $j^*$ such that for all $i\in[n]$,
%\begin{align*}
$
    \Pr[j^*=i] = C_0\cdot\frac{ \exp(y_i)}{ \langle \exp(y), {\bf 1}_n \rangle}. 
$
%\end{align*}
\end{definition}

We first recall the (two-party) set-disjointness communication problem $\mathsf{SetDisj}_n$, in which two parties Alice and Bob have subsets $A$ and $B$, respectively, of $[n]$. 
Note that we can equivalently view $A$ and $B$ as binary vectors in $n$-dimensional space, serving as the indicator vector for whether each index $i\in[n]$ is in the player's input subset. 
The task for the players is to determine whether there exists a common element in their intersection, i.e., whether there exists $i\in[n]$ such that $i\in(A\cap B)$ or equivalently, $A_i=B_i=1$. 
In fact, the problem promises that either the inputs are completely disjoint, $|A\cap B|=0$ or the inputs contain only a single coordinate in their intersection, $|A\cap B|=1$. 
We recall the following standard communication complexity result of set-disjointness. 
\begin{theorem}[\cite{KalyanasundaramS92,Razborov92,Bar-YossefJKS04}]
\thmlab{thm:set:disjointness}

Any protocol that solves the set-disjointness problem $\mathsf{SetDisj}_n$ with probability at least $\frac{3}{4}$ requires $\Omega(n)$ bits of total communication. 
\end{theorem}
We show that even a sampler that relaxes the probability distribution defined in \defref{def:exponential_sampler} up to a factor of $n^C$ is infeasible in the streaming model. 
\begin{theorem}\thmlab{thm:softmax_lower_bound}
Let $y\in\mathbb{R}^n$ that arrives as a data stream and let $C>0$ be a constant. 
Then any algorithm that samples an index $i\in[n]$ with probability proportional to $p_i = \frac{ \exp(y_j)}{ \langle \exp(y), {\bf 1}_n \rangle }$ must use $\Omega(n)$ bits of space, even if the sampling probabilities are allowed to be distorted by as large as $n^C$ and even if $\|y\|_\infty=O(\log n)$. 
%Let us assume that $\| y \|_{\infty} = O(\log n)$. 
%There is no algorithm that uses $o(n)$ space to generate a exponential sampler (Definition~\ref{def:exponential_sampler}) with respect $\langle \exp(y) , {\bf 1}_n \rangle^{-1} \exp(y)$.
\end{theorem}
\begin{proof}
Let $A,B\in\{0,1\}^n$ be input vectors from the set disjointness problem, so that the goal is to determine whether there exists $i\in[n]$ such that $A_i=B_i=0$. 
Observe that Alice and Bob can multiply $A$ and $B$ by $100C\log n$ for some constant $C>0$. 
Now, note that in the disjoint case, we have that $\|A+B\|_{\infty} = 100C \log n$ and in the non-disjoint case, we have that $\|A+B\|_{\infty} = 200C \log n$. 
In particular, in the non-disjoint case, there exists $i\in[n]$ such that $A_i+B_i=200C\log n$ and for all $j\neq i$, we have that $A_j+B_j\le 100C\log n$. 
Hence, in the non-disjoint case, any exponential sampler will output $i$ with probability proportional to $\exp(200C\log n)$ and output $j\neq i$ with probability proportional to $n\cdot\exp(100C\log n)$. 
Even if the sampling probabilities are distorted by a factor of $n^C$, any exponential sampler would output $i$ with probability at least $\frac{3}{4}$. 

Thus, Alice and Bob can use such a data structure to sample an index $i$ and then check whether $A_i=B_i=1$. 
In particular, Alice can first create a data stream encoding the vector $A$, run the sampling algorithm on the data stream, and then pass the state of the algorithm to Bob. 
Bob can then create another portion of the data stream encoding an addition of the vector $B$, take the state of the algorithm from Alice, run the sampling algorithm on the portion of the data stream, and query the algorithm for an index $i$. 
Bob can then take the index and pass it to Alice, and the two parties can finally communicate whether $A_i=B_i=1$, thereby solving set-disjointness with probability at least $\frac{3}{4}$. 
Note that the communication of the protocol is the space used by the sampling algorithm. 
Therefore by \thmref{thm:set:disjointness}, such a sampler must use $\Omega(n)$ bits of space. 
\end{proof}

\section{\texorpdfstring{$\ell_2$}{L2} Sampler Upper Bound with \texorpdfstring{$A$}{A} and \texorpdfstring{$x$}{x}}\seclab{sec:upper_A_x}
In this section, we describe a standard data structure for $\ell_2$ sampling. 
We start with providing the definition of $\ell_2$ sampler as follows,
\begin{definition}
Let $n$ denote a positive integer. Let $\epsilon \geq 0$ denote a parameter. 
In $\ell_2$ sampling, we receive each coordinate of $y \in \mathbb{R}^n$ in a turnstile data stream, and the goal is to output an index $I\in[n]$ at all times such that for each $j\in[n]$,
$
    \Pr[I=j] =  (1\pm \epsilon) \cdot \frac{ | y_j |^2 }{ \| y \|_2^2 } + 1/\poly(n).
$
 
\end{definition}

We introduce various instantiations of the $\ell_2$ sampler for sampling entries from a vector $Ax\in\mathbb{R}^n$, based upon whether the matrix $A\in\mathbb{R}^{n\times d}$ is updated during the data stream, whether the vector $x\in\mathbb{R}^d$ is updated during the data stream, or both. 
To begin with, we review the standard $\ell_2$ sampler in the streaming setting.

\subsection{\texorpdfstring{$\ell_2$}{L2} Sampler}\seclab{sec:background_on_ell_2}

We give the full details of the standard $\ell_2$ sampler from \cite{JowhariST11,MahabadiRWZ20} in \algref{alg:sampler}. 
In this context, the goal is to sample a coordinate from $y\in\mathbb{R}^n$ with probability proportional to $|y_i|^2$, up to $\frac{1}{\poly(n)}$ factors. 
The main intuition is that if $u_i\in[0,1]$ is a uniform random variable, then $\mathbf{Pr}\left[\frac{y_i^2}{u_i}\ge\|y\|_2^2\right]$ is precisely $\frac{y_i^2}{\|y\|_2^2}$. 
If we can identify this case and return $i\in[n]$, then the sampling distribution roughly matches the $\ell_2$ sampling probability distribution. 
Of course, there are various complications such as computing the quantities $y_i^2$ and $\|y\|_2^2$, as well as ensuring exactly one index $i\in[n]$ satisfies $\frac{y_i^2}{u_i}\ge\|y\|_2^2$, but these can all be handled by standard approaches. 
Indeed, the proof of correctness is verbatim from \cite{JowhariST11,MahabadiRWZ20}. 
The challenge is how to implement the data structures of $y$, which is implicitly defined as $Ax$. 
By comparison, in the standard setting of $\ell_2$ samplers~\citep{MonemizadehW10,AndoniKO11,JowhariST11,MahabadiRWZ20,JayaramW21,JayaramWZ22,SwartworthWZ25,WoodruffXZ25}, $y$ is given as a data stream. 

\begin{algorithm}[!htb]
\caption{Standard $\ell_2$ Sampler, e.g., extension of \cite{JowhariST11} to $p=2$}
\alglab{alg:sampler}
\begin{algorithmic}[1]
\State{For each $i\in[n]$, let $u_i\in[0,1]$ be chosen uniformly at random}
\State{$w_i\gets\frac{y_i}{\sqrt{u_i}}$}
\State{Let $z$ denote the tail vector of $w$ without the largest $\frac{1}{\epsilon^2}$ entries in magnitude}
\State{Let $\widehat{Y}$ be a $2$-approximation of $\|y\|_2$}
\State{Let $\widehat{Z}$ be a $2$-approximation of $\|z\|_2$}
\State{$i\gets\text{argmax}_{i\in[n]}|\widehat{w_i}|$}
\State{Let $C>0$ be a large constant determined by the additive failure probability $\frac{1}{\poly(n)}$}
\If{$\widehat{Z}>\sqrt{\frac{C\log n}{\epsilon}}\cdot\widehat{Y}$ or $|w_i|<\sqrt{\frac{C\log n}{\epsilon}}\cdot\widehat{Y}$}
\State{Return FAIL}
\Else
\State{Return $i$ with estimate $\sqrt{u_i}\cdot\widehat{w_i}$}
\EndIf
\end{algorithmic}
\end{algorithm}

\subsection{\texorpdfstring{$A$}{A} is updating during the streaming and \texorpdfstring{$x$}{x} is fixed}
In this section, we describe the construction of an $\ell_2$ sampler for sampling coordinates of the vector $Ax\in\mathbb{R}^n$, in the setting where the vector $x\in\mathbb{R}^d$ is fixed, but the entries of $A\in\mathbb{R}^{n\times d}$ are evolving as the data stream progresses. 

\begin{definition}[Updating $A$ and fixed $x$]
\deflab{def:update_A_fix_x}
In this setting, we assume $x \in \R^d$ is fixed, we receive updates to the entries of $A \in \R^{n \times d}$ in a turnstile data stream. 
Then for $y = Ax$, we want a data structure that produces the $\ell_2$ sampling guarantee for $y$.
\end{definition}
We remark that a turnstile data stream means that each update of the data stream can increase or decrease a single entry of $A$. 

In this work, we are interested in the regime of $n \gg d$.
Then we have the following guarantee:
\begin{theorem}\thmlab{thm:update_A_fix_x}
Suppose $y=Ax$, for $x\in\mathbb{R}^n$, which is fixed, and $A\in\mathbb{R}^{n\times d}$, which is defined by a turnstile stream.  
There exists an $\ell_2$-attention sampler that uses $d\log n+\poly\left(\frac{1}{\epsilon},\log n\right)$ bits of space and returns $I\in[n]$ such that $\Pr[I=j] =  (1\pm \epsilon) \cdot \frac{ | y_j |^2 }{ \| y \|_2^2 } + 1/\poly(n)$. 
The update time of the data structure is $d\,\poly\left(\frac{1}{\epsilon},\log n\right)$.  
%Let $m = \epsilon^{-2} \log n$. 
%There is a streaming algorithm that uses only $O(m+ d)$ to generate a $\ell_2$ sampler. Each update time is $O(1)$.
\end{theorem}

\begin{proof}
Recall that existing approximate $\ell_2$ samplers, e.g., \algref{alg:sampler} maintains a linear sketch $\Phi y$, where $\Phi\in\mathbb{R}^{m\times n}$, for $m=\poly\left(\frac{1}{\epsilon},\log n\right)$. 
We have $y=Ax$, where $x\in\mathbb{R}^d$ is fixed but $A\in\mathbb{R}^{n\times d}$ is defined through turnstile updates. 
Nevertheless, we can maintain the state of $\Phi Ax$. 
In particular, whenever we receive an update in $A_{i,j}$ by $\Delta$, then we can compute $\Phi e_i e_j^\top \Delta x$ to update the sketch $\Phi Ax$. 
To analyze the space complexity, observe that storing $\Phi Ax$ requires $O(m)$ words of space and $x$ requires $d$ words of space, which is $d\log n+\poly\left(\frac{1}{\epsilon},\log n\right)$ bits of space in total. 
Moreover, each update to $A_{i,j}$ can change all entries of $\Phi Ax$, so the update time is $O(md)=d\,\poly\left(\frac{1}{\epsilon},\log n\right)$. 
%Let $\Phi \in \R^{m \times n}$ denote some sketching matrices.
%
%Whenever we get an update in $A_{i,j}$ by $\Delta$, then we just need to compute
%\begin{align*}
    %\Phi e_i e_j^\top \Delta x
%\end{align*}
%The update time of this process is just update time/column sparsity of $\Phi$. (In Fact, we can choose $\Phi$  to be a countsketch matrix, in this sense, the column sparsity of $\Phi$ is $O(1)$).
%
%The dependence of $m$ is essentially $m = O( \epsilon^{-2} \log n )$ if we want $(1\pm\epsilon)$ approximation.
%
%The space is coming three parts:
%\begin{itemize}
    %\item store $x$, this is space $O(d)$
    %\item implicitly store sketch $\Phi$
    %\item store $\Phi A x \in \R^m$, this is space $O(m)$
%\end{itemize}

\end{proof}

\subsection{\texorpdfstring{$x$}{x} is updating during the streaming and \texorpdfstring{$A$}{A} is fixed}
We next consider the setting where the vector $x\in\mathbb{R}^d$ is updated as the data stream progresses, but the entries of $A\in\mathbb{R}^{n\times d}$ are fixed.
\begin{definition}[Fixed $A$ and updating $x$]
\deflab{def:fix_A_update_x}
We assume $A \in \R^{n \times d}$ is fixed, we receive updates to $x \in \R^{ d}$ in a turnstile data stream. 
Then for $y = Ax$, we want a data structure that produces the $\ell_2$ sampling guarantee for $y$.
\end{definition}
We have the following algorithmic guarantees for this setting:
\begin{theorem}\thmlab{thm:fix_A_update_x}
%Let $m = O(\epsilon^{-2} \log n)$.
%There is a streaming algorithm that solves Definition~\ref{def:fix_A_update_x} and uses $O(md)$ spaces. Each update time is $O(m)$
Suppose $y=Ax$, for $A\in\mathbb{R}^{n\times d}$, which is fixed, and $x\in\mathbb{R}^n$, which is defined by a turnstile stream.  
There is an $\ell_2$-attention sampler that uses $d\,\poly\left(\frac{1}{\epsilon},\log n\right)$ bits of space and returns $I\in[n]$ such that $\Pr[I=j] =  (1\pm \epsilon) \cdot \frac{ | y_j |^2 }{ \| y \|_2^2 } + 1/\poly(n)$. 
The update time of the data structure is $O(1)$. 
%and the post-processing time is $O(md^{\omega-1})$, where $\omega$ is the matrix multiplication exponent. 
\end{theorem}
\begin{proof}
Again recall that existing approximate $\ell_2$ samplers, e.g., \algref{alg:sampler} maintains a linear sketch $\Phi y$, where $\Phi\in\mathbb{R}^{m\times n}$, for $m=\poly\left(\frac{1}{\epsilon},\log n\right)$. 
Since $y=Ax$, but $A\in\mathbb{R}^{n\times d}$ is too large to store, while $x\in\mathbb{R}^n$ is defined through turnstile updates, we can instead maintain the sketch $\Phi A$ and the vector $x$ and compute $\Phi Ax =\Phi y$ after the stream concludes. 
Note that storing $\Phi A$ requires $O(md)$ words of space and $x$ requires $d$ words of space, which is $d\,\poly\left(\frac{1}{\epsilon},\log n\right)$ bits of space in total. 
Moreover, each update to $x$ changes a single entry, so the update time is $O(1)$. 
%, while the post-processing time is $O(md^{\omega-1})$, where $\omega$ is the matrix multiplication exponent, due to the computing the matrix-vector product $(\Phi A)x$. 
\end{proof}

\subsection{Both \texorpdfstring{$A$}{A} and \texorpdfstring{$x$}{x} are updating during the streaming}
Finally, we consider the setting where both the vector $x\in\mathbb{R}^d$ and the entries of $A\in\mathbb{R}^{n\times d}$ can be changed by updates from the data stream. 
\begin{definition}[Updating $A$ and updating $x$]
\deflab{def:update_A_update_x}
In this setting, we receive updates to both $A \in \R^{n \times d}$ and $x \in \R^{ d}$ in a turnstile data stream. 
Then for $y = Ax$, we want a data structure that provides the $\ell_2$ sampling guarantee for $y$.
\end{definition}
We have the following guarantees:
\begin{lemma}[Upper Bound]\thmlab{thm:update_A_update_x}
%Let $m = d\,\poly\left(\frac{1}{\epsilon},\log n\right)$, there is a streaming algorithm that solves problem defined as Definition~\ref{def:update_A_update_x}. Each update time is ???.
Suppose $y=Ax$, for $A\in\mathbb{R}^{n\times d}$ and $x\in\mathbb{R}^d$, which are each defined in a stream through turnstile updates. 
There exists an $\ell-2$-attention sampler that uses $d\,\poly\left(\frac{1}{\epsilon},\log n\right)$ bits of space and returns $I\in[n]$ such that $\Pr[I=j] =  (1\pm \epsilon) \cdot \frac{ | y_j |^2 }{ \| y \|_2^2 } + 1/\poly(n)$. 
The update time is $\poly\left(\frac{1}{\epsilon},\log n\right)$. 
%and the post-processing time is $O(md^{\omega-1})$, where $\omega$ is the matrix multiplication exponent. 
\end{lemma}
\begin{proof}
As before, recall that existing approximate $\ell_2$ samplers, e.g., \algref{alg:sampler} maintains a linear sketch $\Phi y$, where $\Phi\in\mathbb{R}^{m\times n}$, for $m=\poly\left(\frac{1}{\epsilon},\log n\right)$. 
Since $y=Ax$, but now both $A\in\mathbb{R}^{n\times d}$ and $x\in\mathbb{R}^n$ are defined through turnstile updates, we can instead maintain the sketch $\Phi A$ and the vector $x$ and compute $\Phi Ax =\Phi y$ after the stream concludes. 
Observe that maintaining $\Phi A$ requires $O(md)$ words of space and $x$ requires $d$ words of space, which is $d\,\poly\left(\frac{1}{\epsilon},\log n\right)$ bits of space in total. 
Each update to $A$ can change all $m$ entries of in a single column of $\Phi A$, while each update to $x$ changes a single entry. 
Hence, the update time is $\poly\left(\frac{1}{\epsilon},\log n\right)$. 
%while the post-processing matrix-vector multiplication time is $O(md^{\omega-1})$, where $\omega$ is the matrix multiplication exponent. 
\end{proof}

\section{\texorpdfstring{$\ell_2$}{L2} Sampler Lower Bound with \texorpdfstring{$A$}{} and \texorpdfstring{$x$}{}}\seclab{sec:lower_A_x}
In this section, we give lower bounds for $\ell_2$ sampling from a vector $y=A^{\otimes p}x$, when either $A$ or $x$ are updated in a data stream. 
We show that in any of these cases, the general problem is substantially more difficult than the previous case where $p=1$. 

We first recall the Index problem for one-way communication. 
In the $\mathsf{INDEX}_n$ problem, Alice receives a vector $v\in\{0,1\}^n$ and Bob receives a coordinate $i\in[n]$. 
The goal is for Bob to compute $v_i$ with probability at least $\frac{3}{4}$, given some message $\Pi$ from Alice. 
We recall the following communication complexity lower bounds for Index. 

\begin{theorem}[\cite{KremerNR99}]
\thmlab{thm:index}
Any protocol that solves $\mathsf{INDEX}_n$ with probability at least $\frac{3}{4}$ requires $\Omega(n)$ bits of communication. 
\end{theorem}

\begin{lemma}[Lower Bound]\thmlab{thm:update_A_update_x_lower_bound}
Any streaming algorithm that solves problem defined as \defref{def:update_A_update_x} will require $\Omega(d)$ space.
\end{lemma}
\begin{proof}
Suppose Alice receives a vector $v\in\{0,1\}^d$. 
Then Alice creates the diagonal matrix $M\in\{0,1\}^{d\times d}$ so that the $j$-th diagonal entry of $A$ is $v_j$, for all $j\in[n]$. 
Finally, Alice creates $A\in\mathbb{R}^{(d+1)\times d}$ by appending the row consisting of $\frac{1}{10^{10}}$ in all of its $d$ entries to $M$. 
Suppose Bob receives the coordinate $i\in[d]$ and wants to determine $v_i$. 
Then Bob can set $x$ to be the elementary vector $e_i\in\mathbb{R}^d$, which has a $1$ in its $i$-th coordinate and zeros elsewhere. 
Observe that by construction, $Ax$ is the $i$-th column of $A$. 
If $v_i=1$, then the $i$-th column of $A$ consists of a $1$ in the $i$-th entry, $\frac{1}{10^{10}}$ in the $(d+1)$-st entry, and zeros elsewhere. 
Hence, a sampler with the desired properties will output $i$ with probability at least $\frac{3}{4}$. 
Similarly, if $v_i=0$, then the $i$-th column of $A$ consists of $\frac{1}{10^{10}}$ in the $(d+1)$-st entry and zeros elsewhere. 
Thus, the sampler with the desired properties will output $d+1$ with probability $1$. 
Bob can therefore distinguish between these two cases with probability at least $\frac{3}{4}$, thereby solving $\mathsf{INDEX}_d$ with probability at least $\frac{3}{4}$. 
Therefore, by \thmref{thm:index}, such a sampler must use at least $\Omega(d)$ space. 
\end{proof}

In fact, we show that if $y=A^{\otimes p}x$, where $A\in\mathbb{R}^{n\times n}$ so that $A^{\otimes p}\in\mathbb{R}^{n^p\times n^p}$ denotes the $p$-wise self-tensor and $x\in\mathbb{R}^{n^p}$, then actually $\ell_2$ sampling from $y$ uses $\Omega(n)$ bits of space. 

\begin{lemma}
Let $A\in\mathbb{R}^{n\times n}$ and $A^{\otimes p}\in\mathbb{R}^{n^p\times n^p}$ denote the $p$-wise self-tensor. 
Let $y=A^{\otimes p}x$, so that $x\in\mathbb{R}^{n^p}$. 
Then even if all the entries of $x$ arrive in a data stream followed by all the entries of $A$, $\ell_2$ sampling from $y$ requires $\Omega(n)$ bits of space. 
\end{lemma}
\begin{proof}
Let $S\in\{0,1\}^{n}$ be an instance of $\mathsf{INDEX}_{n}$. 
Suppose Alice creates the diagonal matrix $A$ with exactly $S$ being the entries across its diagonal, i.e., $A_{1,1}=S_1,\ldots,A_{n,n}=S_n$. 
Bob has an index $i\in[n]$, and sets the vector $x$ to be the elementary vector $\mathbf{e}_j$, where $j=i\cdot n^{p-1}$. 
Then by construction $Ax$ is the all zeros vector if $S_i=0$ and otherwise there is a nonzero entry, which allows Alice and Bob to solve $\mathsf{INDEX}_{n}$. 
Hence, $\ell_2$ sampling from $y$ requires $\Omega(n)$ bits of space. 
\end{proof}

\section{The Tensor Version Problem}\seclab{sec:tensor}
In this section, we further consider sampling from a tensor product. 
We provide the tensor notations and objects.
\begin{definition}
Let $A_1 \in \R^{n \times d}$, let $A_2 \in \R^{n \times d}$, we define $\mathsf{A} = A_1 \otimes A_2 \in \R^{n^2 \times d^2}$.
Let $x \in \R^{d^2}$. Let $\mathsf{A}_i \in \R^{n \times d^2}$ denote the $i$-th block of $\mathsf{A}$.
\end{definition}

%\subsection{Setup}

\begin{definition}[fixed $x$, Streaming Sampler for one of $A_1$ and $A_2$ is updating.]
We assume $x \in \R^{d^2}$ is fixed. We assume that (1)
%\begin{itemize}
%    \item 
    one of $A_1$ and $A_2$ is updating, (2)
    %\item 
    one of $A_1$ and $A_2$ is fixed. 
%\end{itemize}
Let $y = \A x$, we want $\ell_2$ sampling guarantee for sampling one coordinate in $y_i \in \R^{n^2}$ for all $i \in [n^2]$.
\end{definition}

To motive this model, recall that the tensor product $\left(A_1 \otimes A_2\right) x$ equals to the linear cross-attention matrix $A_1 Q K^{\top} A_2^{\top}$, where $W_Q=A_1 Q$ is the projected query matrix and $W_K=A_2 K$ is the projected key matrix. 
Our model addresses a practical scenario involving real-time contextual processing with a static reference dataset. 
In this setting, $W_k$ is precomputed by the language model, representing a static dataset such as embeddings of a knowledge base, user profiles, or multimedia features. 
Then, the rows of matrix $A_1$ arrive as a data stream, representing real-time data queries. Thus, our attention sampler efficiently captures the important entries in the dynamic query dataset.

We use the following formulation of Nisan's pseudorandom generator to derandomize our algorithm. 
\begin{theorem}[Nisan's PRG, \cite{Nisan92}]
Suppose $\mathcal{A}$ is an algorithm that requires $S=\Omega(\log n)$ bits of space and $R$ random bits. 
Then there exists a pseudorandom generator for $\mathcal{A}$ that succeeds with probability $1-{1}/{\poly(n)}$ and uses $O(S\log R)$ bits of space. 
\end{theorem}

\begin{algorithm}[!htb]\caption{We build on algorithm based on $S (x_1 \otimes x_2)$}
\begin{algorithmic}[1]
\State{Suppose we use $O(nd)$ space to store $A_1$ and $A_2$ (Avoid $n^2$ time/space)}
\State Suppose we receive an update $q \in [2]$, $i \in [n], j \in [d], \Delta$
    \State Suppose we have hash function $g$ to access uniform number
    \If{$q = 1$}
        \State $p \gets g(i(n-1)+1, \cdots, i n) $ \Comment{$p \in \R^n$}
        \State $y \gets y + \Phi \Delta (e_{[i(n-1)+1, in]} \circ (A_2)_{*,j} )/p  $ \Comment{$\Phi_1$ is decided by $h_1, \sigma_1$}
    \Else
        \State $y_2 \gets y_2 + \Phi_2 e_i \Delta$ \Comment{$\Phi_2$ is decided by $h_2, \sigma_2$}
    \EndIf
\end{algorithmic}
\end{algorithm}

In the following Lemma, we state a streaming algorithm to solve tensor related sampling problem. We consider the situation that one of $A_1$ and $A_2$ is fixed, and the other one is updated in streaming fashion. 
We show the following estimation guarantees using the standard CountSketch analysis, c.f.,~\cite{CharikarCF04,JowhariST11}. 
\begin{lemma}[Tensor $\ell_2$ Tail Estimation]
\lemlab{lem:tensor_version_error}
Let $y=(A_1\otimes A_2)x \in \mathbb{R}^{n^2}$. Let only one of $A_1$ and $A_2$ be updated in streaming. 
Let $w = \frac{y_i}{ \sqrt{u_i} }$ for a constant $u_i \in [0,1]$ generated uniformly at random. 
There is an algorithm $\mathcal{A}$ that that uses $O(nd) + \poly\left(\frac{1}{\epsilon},\log n\right)$ space, uses $O(n)$  update time, and estimates each element of $w$ up to additive error $\epsilon\cdot\|z\|_2$, where $z$ denotes the tail vector of $w$ without the largest $\frac{1}{\epsilon^2}$ entries in magnitude. 
Specifically, for all $i\in[n^2]$, we have $|\widehat{w}_i - w_i| \leq \epsilon\cdot\|z\|_2$. 
%\begin{align*}
%|\widehat{w}_i - w_i| \leq \epsilon\cdot\|z\|_2,
%\end{align*}
%for all $i\in[n^2]$. 
  
\end{lemma}
\begin{proof}
Consider hash function $h_1, h_2 : [n]\to[b]$. 
Consider random sign functions $\sigma_1, \sigma_2 : [n] \rightarrow \{-1,+1\}$. 
We consider a fixed index $i_1, i_2 \in[n]$. 
Let $j = h_1(i_1) + h_2(i_2) \pmod b$. 
Let $h^{-1}(j)$ denote the all the pairs $(i_1,i_2) \in [n] \times [n]$ such that  $ h_1(i_1) + h_2(i_2) \pmod b = j$.
Note that $\widehat{y}_i$ induced by $h$ is 
%\begin{align*}
$
    \widehat{w}_i = w_i + \sum_{l  \in h^{-1}( j ) \backslash \{i \} }s_i s_{l} w_{l_1}  w_{l_2}
$. 
%\end{align*}
For ease of presentation, we write $\sigma_i = \sigma_{1,i_1} \sigma_{2,i_2}$ and $\sigma_l = \sigma_{1,l_1} \sigma_{2,l_2}$.
\begin{align*}
\E [\widehat{w}_i ]
= & ~ \E \Big[ w_i+\sum_{l \in h^{-1}( j ) \backslash \{i\} } \sigma(i) \sigma(l) w_l \Big]
=  \E[  w_i ] +  \sum_{l \in h^{-1}( j ) \backslash \{i\} } \E[ \sigma ( i ) \cdot \sigma ( l ) ] \cdot w_l  \\
= & ~ w_i + \sum_{l \in h^{-1}( j ) \backslash \{i\} } \E[ \sigma( i ) ] \cdot \E[ \sigma( l ) ] \cdot w_l = w_i,
\end{align*}
where the first step follows from definition, the second step follows from linearity of expectation, the third step follows from $\sigma( i )$ and $\sigma( l )$ are independent, the forth step follows from  $\E[ \sigma( l ) ]=0$. 

We now upper bound the variance of $\widehat{w}_i-y_i$ by analyzing $\E [(\widehat{y}_i )^2 ]$. 
Let $\mathcal{H}$ be the set of the top $\frac{1}{\epsilon^2}$ items and let $\mathcal{E}$ be the event that none of the items in $\mathcal{H}$ are mapped to $h(i)$, i.e., $h(a)\neq h(i)$ for all $a\in\mathcal{H}$.

Observe that for $b=\frac{100}{\epsilon^2}$, we have that $\Pr[\mathcal{E}]\ge 0.9$. Then we have:
\begin{align*}
 \E[(\widehat{w}_i-w_i)^2~|~\mathcal{E} ] 
= & ~ \E [ ( \sum_{l \in [n]^2 \setminus\mathcal{H}, l \in h^{-1}(j) } \sigma(i) \sigma(l) w_l )^2 ]
=  \E\left[ \sum_{l \in[n]^2 \backslash \mathcal{H}, l \in h^{-1}(j) } w_l^2 \right]\\
= & ~ \frac{1}{b}\cdot \sum_{l \in[n]^2 \backslash \mathcal{H}, l \in h^{-1} (j) } w_l^2 
\leq  \frac{1}{b}\cdot (w_1^2+\ldots+w_{n^2}^2-\sum_{l \in\mathcal{H}} w_l^2 )\\
= & ~ 100 \epsilon^2 \cdot \|z \|_2^2,
\end{align*}
for $b=\frac{100}{\epsilon^2}$, since $z$ is the vector corresponding to $y$ that removes the entries in $\mathcal{H}$. 
By Chebyshev's inequality, we have that 
%\begin{align*}
$
\Pr [ |\widehat{w}_i -w_i| \ge \epsilon\cdot\|z\|_2~|~\mathcal{E} ] \le\frac{1}{10}.
$
%\end{align*}
Since $\Pr[\mathcal{E}]\ge 0.9$, then
$\Pr{|\widehat{w}_i-w_i|\ge\epsilon\cdot\|z\|_2}\le0.2,$ 
for a fixed hash function $h$. 
By taking the median of $O(\log n)$ estimations corresponding to $O(\log n)$ different hash functions $h$, we have that
%\begin{align*}
$
\Pr [ |\widehat{w}_i-w_i|\ge\epsilon\cdot\|z\|_2 ] \le\frac{1}{n^{10}}.
$
%\end{align*}
Thus by a union bound over $i\in[n] \times [n]$, we have that with probability at least $1-\frac{1}{n^5}$, we have for all $i\in[n]$,
%\begin{align*}
$
|\widehat{w}_i - w_i| \ge \epsilon \cdot \|z\|_2.
$
%\end{align*}
\end{proof}

%\subsection{From Tail to Sampling}
We state the following lemma as a structural property that will allow us to achieve our tensor product sampler. 
We remark that the proof is extended from that of approximate $\ell_p$ sampling \citep{JowhariST11}. 
%Thus we defer the proof to Appendix~\ref{sec:tail_to_sampling}. 
\begin{lemma}
\lemlab{lem:sampler:bad}
Let $y=(A_1\otimes A_2)x\in\mathbb{R}^{n^2}$ and let $w\in\mathbb{R}^{n^2}$ so that $w_i=\frac{y_i}{\sqrt{u_i}}$ for a constant $u_i\in[0,1]$ generated uniformly at random. 
Let $z$ denote the tail vector of $w$ without the largest $\frac{1}{\epsilon^2}$ entries in magnitude. 
Let $\widehat{Z}$ be a $2$-approximation to $\|z\|_2$ and $\widehat{Y}$ be a $2$-approximation to $\|y\|_2$, then we have $\Pr[\widehat{Z}>\sqrt{ (C\log n) /\epsilon}\cdot\widehat{Y}]\le O(\epsilon)+\frac{1}{\poly(n)}$.
\end{lemma}
\begin{proof}
Let $\mathcal{E}_1$ denote the event that $\widehat{Z}$ is a $2$-approximation to $\|z\|_2$ and $\widehat{Y}$ is a $2$-approximation to $\|y\|_2$, so that 
\begin{align*}
\Pr[\mathcal{E}_1]\ge1-\frac{1}{\poly(n)}.
\end{align*}

Conditioned on $\mathcal{E}_1$, it suffices to bound the probability that 
\begin{align*}
4\|z\|_2>\sqrt{\frac{C\log n}{\epsilon}}\cdot\|y\|_2.
\end{align*}
Let $j\in[n^2]$ be a fixed index and let $u_j$ be fixed. 

Let $T=\sqrt{\epsilon}\cdot\|y\|_2$ and for each $i\in[n^2]$, we define the indicator random variable $W_i=1$ if $|w_i|>T$ and $W_i=0$ otherwise, if $|w_i|\le T$. 
Note that $W_i$ is an indicator random variable for whether the coordinate $w_i$ in the vector $w$ is ``heavy'' in magnitude.

We then define 
\begin{align*}
    Z_i=\frac{w_i^2}{T^2}\cdot(1-W_i)
\end{align*}
to be the scaled contribution of the small entries of $z$, and observe that $Z_i\in[0,1]$.

Let 
\begin{align*}
W = \sum_{i\in[n^2], i\neq j} w_i
\end{align*}
denote the total number of heavy indices besides possibly index $j$ and $Z=\sum_{i\in[n^2], i\neq j}Z_i$ denote the total scaled contribution of the light indices besides possibly index $j$. 
Let $v$ denote the vector containing the heavy indices, so that $v_i=w_i$ for $W_i=1$ and $v_i=0$ otherwise for $W_i=0$. 
Note that $v$ has sparsity at most $Y+1$ and moreover $U^2Z=\|w-v\|_2^2$. 
We also have that $\|z\|_2\le\|w-v\|_2$ unless $W\ge\frac{2}{\epsilon^2}$. 

Let $\mathcal{E}_2$ denote the event that $W\ge\frac{2}{\epsilon^2}$ and let $\mathcal{E}_3$ denote the event that $Z\ge\frac{C\log n}{16T^2\epsilon}\cdot\|y\|_2^2$. 
Observe that if neither $\mathcal{E}_2$ nor $\mathcal{E}_3$ occur, then we have $4\|z\|_2\le\sqrt{\frac{C\log n}{\epsilon}}\cdot\|y\|_2$, as desired. 
Thus it remains to bound the probability of the failure events $\mathcal{E}_2$ and $\mathcal{E}_3$. 

We have $\E[W_i]=\frac{\|w\|_2^2}{T^2}$, so that $\E[W]\le\frac{1}{\epsilon}$. 
By Markov's inequality, we have that $\Pr[\mathcal{E}_2]\le\frac{\epsilon}{2}$. 

We now upper bound $\Pr[\mathcal{E}_3]$. 
Recall that $Z_i=\frac{w_i^2}{T^2}\cdot(1-W_i)=\frac{w_i^2}{Tu_i^2}\cdot(1-W_i)$, since $w_i=\frac{y_i}{\sqrt{u_i}}$. 
Observe that $Z_i>0$ only if $|w_i|<T$, i.e., if $u_i\ge\frac{y_i^2}{\epsilon\cdot\|y\|_2^2}$, since $T=\sqrt{\epsilon}\cdot\|y\|_2$. 
For $\epsilon\in(0,1)$, we thus have
\begin{align*}
\E[Z_i]
\leq & ~ \int_{y_i^2/\|y\|_2^2}^1 z_i \d u_i \\
= & ~ \int_{y_i^2/\|y\|_2^2}^1\frac{y_i^2}{u_i}\frac{1}{T^2} \d u_i.
\end{align*}
Now, let $\mathcal{E}_4$ be the event that $u_i\ge\frac{1}{n^{C/2}}$ for all $i\in[n^2]$, so that $\Pr[\mathcal{E}_4]\ge 1-\frac{1}{n^{C/2-2}}$.

Then
\begin{align*}
\E [Z_i ~|~ \mathcal{E}_4] 
\leq & ~ \frac{1}{1-\frac{1}{n^{C/2-2}}}\int_{1/n^{C/2}}^1\frac{y_i^2}{u_i}\frac{1}{T^2} \d u_i \\
\leq & ~ \frac{C\log n}{T^2}y_i^2.
\end{align*}
Thus, we have
\begin{align*}
\E[Z~|~\mathcal{E}_4]
= & ~ \sum_{i\in[n^2]}\E[Z_i~|~\mathcal{E}_4]\\
= & ~ \sum_{i\in[n^2]}\frac{C\log n}{T^2}y_i^2\\
\le & ~ \sum_{i\in[n^2]}\frac{C\log n}{\epsilon}\frac{y_i^2}{\|y\|_2^2} \\
= & ~ \frac{C\log n}{\epsilon}.
\end{align*}
Thus by Markov's inequality, the probability that $Z$ is larger than $\frac{C\log n}{16T^2\epsilon}\cdot\|y\|_2^2=\frac{C\log n}{16\epsilon^2}$ is at most $\frac{\epsilon}{16}$. 
The claim then follows from taking a union bound over the events $\mathcal{E}_1,\neg\mathcal{E}_2,\neg\mathcal{E}_3,\neg\mathcal{E}_4$. 
\end{proof}
Finally, we describe the guarantees of our tensor-based sampler.

\begin{theorem}\thmlab{thm:tensor_sampler}
Let $y=(A_1\otimes A_2)x\in\mathbb{R}^{n^2}$ and   let $w\in\mathbb{R}^{n^2}$ so that for each $i \in [n^2]$, $w_i=\frac{y_i}{\sqrt{u_i}}$ for a constant $u_i\in[0,1]$ generated uniformly at random. 
Let $z$ denote the tail vector of $w$ without the largest $\frac{1}{\epsilon^2}$ entries in magnitude. 
Suppose there exists:
\begin{enumerate}
\item 
An algorithm $\mathcal{A}_1$ that provides a $2$-approximation to $\|y\|_2$ with probability $1-\frac{1}{n^2}$. 
\item 
An algorithm $\mathcal{A}_2$ that provides a $2$-approximation to $\|z\|_2$ with probability $1-\frac{1}{n^2}$. 
\item 
An algorithm $\mathcal{A}_3$ that estimates each element of $w$ up to additive error $\epsilon\cdot\|z\|_2$,
$|\widehat{w_i}-w_i|\le\epsilon\cdot\|z\|_2,$
for all $i\in[n^2]$. 
\end{enumerate}
Then there exists a data structure that uses $\poly\left(\frac{1}{\epsilon},\log n\right)$ bits of space and outputs each index $i$ with probability $p_i = \left(1\pm\epsilon\right)\cdot\frac{y_i^2}{\|y\|_2^2}\pm\frac{1}{\poly(n)}$.
%\[\left(1-\epsilon\right)\cdot\frac{y_i^2}{\|y\|_2^2}-\frac{1}{\poly(n)}\le p_i\le\left(1+\epsilon\right)\cdot\frac{y_i^2}{\|y\|_2^2}+\frac{1}{\poly(n)}.\]
\end{theorem}
\begin{proof}
Let $i$ be fixed and let $\mathcal{E}$ denote the event that $u_i<\frac{\epsilon}{C\log n}\frac{y_i^2}{\widehat{Y}^2}$, so that $|w_i|>\sqrt{\frac{C\log n}{\epsilon}}\cdot\widehat{Y}$.

Let $\mathcal{E}_1$ denote the event that $\widehat{Y}$ is a $2$-approximation to $\|y\|_2$, $\widehat{Z}$ is a $2$-approximation to $\|z\|_2$, and $|\widehat{w_i}-w_i|\le\epsilon\cdot\|z\|_2$ for all $i\in[n]$. 
Let $\mathcal{E}_2$ denote the event that $\widehat{Z}>\sqrt{\frac{C\log n}{\epsilon}}\cdot\widehat{Y}$ and let $\mathcal{E}_3$ denote the event that multiple indices $j$ satisfy $|w_j|>\sqrt{\frac{C\log n}{\epsilon}}\cdot\widehat{Y}$. 
Finally, let $\mathcal{E}_4$ denote the event that $|\widehat{w_i}|<\sqrt{\frac{C\log n}{\epsilon}}\cdot\widehat{Y}$. 

Intuitively, $\mathcal{E}_1$ is a good event, i.e., correctness of the data structures, which we would like to hold. 
On the other hand, $\mathcal{E}_2, \mathcal{E}_3, \mathcal{E}_4$ are bad events that distort the sampling probabilities, which we would like to avoid. 

We first note that $\mathcal{E}_1$ holds with high probability due to the correctness of the CountSketch and $\ell_2$-norm estimation data structures. 
We next note that by \lemref{lem:sampler:bad}, the probability that $\mathcal{E}_2$ occurs is $O(\epsilon)$. 

Next, note that the probability that for a fixed $j\in[n]$, $u_j$ satisfies $\frac{y_j^2}{u_j}\ge\frac{C\log n}{\epsilon}\cdot\widehat{Y}$ is at most $\frac{\epsilon}{C'\log n}\frac{y_j^2}{\|y\|_2^2}$ for some constant $C'$. 
Thus summing over all $j\in[n]$, the probability that there exist an additional $j\in[n]$ for which $|w_j|>\sqrt{\frac{C\log n}{\epsilon}}\cdot\widehat{Y}$ is $O(\epsilon)$. 
Thus the probability that $\mathcal{E}_3$ occurs is $O(\epsilon)$. 

Finally, conditioned on $\neg\mathcal{E}_2$, we have that $\widehat{Z}\le\sqrt{\frac{C\log n}{\epsilon}}\cdot\widehat{Y}$. 
Then conditioning on $\mathcal{E}_1$, we have $\|z\|_2\le\widehat{Z}$ and thus $|\widehat{w_i}-w_i|\le\epsilon\widehat{Z}\le\sqrt{C\epsilon\log n}\widehat{Y}$, so that $\mathcal{E}_4$ can only occur for $\sqrt{\frac{C\log n}{\epsilon}}\cdot\widehat{Y}\le|w_i|\le\sqrt{\frac{C\log n}{\epsilon}}\cdot\widehat{Y}$, which is at most probability $O\left(\frac{\epsilon^2}{C\log n}\frac{y_i^2}{\widehat{Y}^2}\right)$, over the randomness of $u_i$. 

In summary, we observe that conditioned on some value being output, the probability that item $i$ is selected is proportional to the event that the events $\mathcal{E}$ and $\mathcal{E}_1$ occur, and none of the events $\mathcal{E}_2,\mathcal{E}_3,\mathcal{E}_4$ occur. 
The probability that $\mathcal{E}$ occurs is $\frac{\epsilon}{C\log n}\frac{y_i^2}{\widehat{Y}^2}$, which $u_i$ is chosen uniformly at random. 
Due to the event $\mathcal{E}_1$, the sampling probability is distorted additively by $\frac{1}{\poly(n)}$, while due to the events $\mathcal{E}_2,\mathcal{E}_3,\mathcal{E}_4$, the sampling probability is distorted multiplicatively by $(1+\epsilon)$. 
Thus conditioned on the event that some index is returned, the probability $p_i$ that index $i$ is returned satisfies 
\[\left(1-\epsilon\right)\cdot\frac{y_i^2}{\|y\|_2^2}-\frac{1}{\poly(n)}\le p_i\le\left(1+\epsilon\right)\cdot\frac{y_i^2}{\|y\|_2^2}+\frac{1}{\poly(n)},\]
as desired.
\end{proof}

We remark that the algorithms $\mathcal{A}_1$ and $\mathcal{A}_2$ in the context of \thmref{thm:tensor_sampler} can be achieved using the standard AMS $\ell_2$ norm estimator~\citep{AlonMS99}. 
Moreover, algorithm $\mathcal{A}_3$ in the context of \thmref{thm:tensor_sampler} can be achieved using the standard CountSketch algorithm~\citep{CharikarCF04}.

\section{Conclusions}
To achieve efficient attention mechanisms, we introduce the attention sampler and study its behavior in the streaming model. 
We established efficient polynomial samplers under various streaming settings, when the input matrix, the weight vector, or both evolve dynamically, and we complement the results by proving space lower bounds. Our framework identify the critical components in attention computation, offering a foundation for efficient simulations of large-scale attention schemes, which is central to modern machine learning and LLMs. 

For future directions, from a theoretical perspective, given the $\Omega(n)$ lower bound on exponential samplers in general circumstances, it would be valuable to explore whether we can achieve $o(n)$ space under certain assumptions, e.g., restricting the entries in the attention matrix to $o(\log n)$. From a practical perspective, it would be beneficial to evaluate our sampler's performance by implementing it in existing sparse attention schemes and streaming attention schemes.


\def\shortbib{0}
\bibliography{ref}
\bibliographystyle{plainnat}

%%%%%%%%%%%%%%%%%%%%%%%%%%%%%%%%%%%%%%%%%%%%%%%%%%%%%%%%%%%%%%%%%%%%%%%%%%%%%%%
%%%%%%%%%%%%%%%%%%%%%%%%%%%%%%%%%%%%%%%%%%%%%%%%%%%%%%%%%%%%%%%%%%%%%%%%%%%%%%%

\end{document}